\documentclass[12pt]{article}
\usepackage{graphicx} 
\usepackage{amsmath, amssymb, amsthm,wasysym}
\usepackage{algorithm}
\usepackage{algorithmic}
\usepackage{color}
\usepackage{times}

\usepackage[font=scriptsize,labelfont=bf]{caption}

\newtheorem{theorem}{Theorem}

\newtheorem{fact}{Fact}

\newcommand{\field}[1]{\mathbb{#1}}

\newcommand{\E}{\field{E}}
\newcommand{\Ind}[1]{\field{I}\bigl\{{#1}\bigr\}}

\newcommand{\bu}{\boldsymbol{u}}
\newcommand{\bv}{\boldsymbol{v}}

\newcommand{\spin}{\{-1,+1\}}

\newcommand{\scO}{\mathcal{O}}

\newcommand{\tp}{\textsc{extractTreelet}}

\newcommand{\path}{\mathrm{Path}}

\renewcommand{\Pr}{\field{P}}

\newcommand{\scP}{\mathcal{P}}
\newcommand{\scS}{\mathcal{S}}

\newcommand{\Eflip}{ E_{\mathrm{flip}} }

\newcommand{\alg}{\textsc{treeCutter}}
\newcommand{\algb}{\textsc{cccc}}

\renewcommand{\path}{\mathrm{Path}}
\newcommand{\chd}{\mathrm{Children}}

\title{A Linear Time Active Learning Algorithm \\ for Link Classification\\
-- Full Version -- \thanks
{This work was supported in part by the PASCAL2 Network of Excellence under EC grant 216886 and by ``Dote Ricerca'', FSE, Regione Lombardia. This publication only reflects the authors' views.}}

\author{
Nicol\`o Cesa-Bianchi\\ 
Dipartimento di Informatica, Universit\`a degli Studi di Milano, Italy\\
\texttt{nicolo.cesa-bianchi@unimi.it}
\and
Claudio Gentile\\ 
DiSTA, Universit\`a dell'Insubria, Italy\\
\texttt{claudio.gentile@uninsubria.it}
\and
Fabio Vitale\\ 
Dipartimento di Informatica, Universit\`a degli Studi di Milano, Italy\\
\texttt{fabio.vitale@unimi.it}
\and
Giovanni Zappella\\
Dipartimento di Matematica, Universit\`a degli Studi di Milano, Italy\\
\texttt{giovanni.zappella@unimi.it}
}

\begin{document} 

\maketitle


\begin{abstract} 
We present very efficient active learning algorithms for link classification in signed networks. Our algorithms are motivated by a stochastic model in which edge labels are obtained through perturbations of a initial sign assignment consistent with a two-clustering of the nodes. We provide a theoretical analysis within this model, showing that we can achieve an optimal (to whithin a constant factor) number of mistakes on any graph $G = (V,E)$ such that $|E| = \Omega(|V|^{3/2})$ by querying $\scO(|V|^{3/2})$ edge labels. More generally, we show an algorithm that achieves optimality to within a factor of $\scO(k)$ by querying at most order of $|V| + (|V|/k)^{3/2}$ edge labels. The running time of this algorithm is at most of order $|E| + |V|\log|V|$.
\end{abstract}

\section{Introduction}
A rapidly emerging theme in the analysis of networked data is the study of signed networks. From a mathematical
point of view, signed networks are graphs whose edges carry a sign representing 
the positive or negative nature of the relationship between the incident nodes. For example, in a protein 
network two proteins may interact in an excitatory or inhibitory fashion. The domain of social networks 
and e-commerce offers several examples of signed relationships: Slashdot users can tag other 
users as friends or foes, Epinions users can rate other users positively or negatively, 
Ebay users develop trust and distrust towards sellers in the network. 
More generally, two individuals that are related because they rate similar products in a recommendation 
website may agree or disagree in their ratings.

The availability of signed networks has stimulated the design of link classification algorithms, especially in the domain of social networks. Early studies of signed social networks are from the Fifties. E.g., \cite{ha53} and~\cite{ch56} model dislike and distrust relationships among individuals as (signed) weighted edges in a graph.
The conceptual underpinning is provided by the theory of {\em social balance}, formulated as a way to understand the structure of conflicts in a network of individuals whose mutual relationships can be classified as friendship or hostility \cite{hei46}. The advent of online social networks has revamped the interest in these theories, and spurred a significant amount of recent work ---see, e.g., \cite{GKRT04,KLB09,LHK10b,cntd11,fia11,cgvz12}, and references therein.

Many heuristics for link classification in social networks are based on a form of social balance summarized by the motto ``the enemy of my enemy is my friend''. This is equivalent to saying that the signs on the edges of a social graph tend to be consistent with some two-clustering of the nodes. By consistency we mean the following: The nodes of the graph can be partitioned into two sets (the two clusters) in such a way that edges connecting nodes from the same set are positive, and edges connecting nodes from different sets are negative. 
Although two-clustering heuristics do not require strict consistency to work, this is admittely a rather strong inductive bias. 
Despite that, social network theorists and practitioners found this to be a reasonable bias in many social contexts, and recent experiments with online social networks reported a good predictive power for algorithms based on the two-clustering assumption \cite{KLB09,LHK10,LHK10b,cntd11}.
Finally, this assumption is also fairly convenient from the viewpoint of algorithmic design.

In the case of undirected signed graphs $G = (V,E)$, the best performing heuristics exploiting the two-clustering bias are based on spectral decompositions of the signed adiacency matrix.
Noticeably, these heuristics run in time $\Omega\bigl(|V|^2\bigr)$, and often require a similar amount of memory storage even on sparse networks, which makes them impractical on large graphs.

In order to obtain scalable algorithms with formal performance guarantees, we focus on the active learning protocol, where training labels are obtained by querying a desired subset of edges. Since the allocation of queries can match the graph topology, a wide range of graph-theoretic techniques can be applied to the analysis of active learning algorithms. In the recent work~\cite{cgvz12}, a simple stochastic model for generating edge labels by perturbing some unknown two-clustering of the graph nodes was introduced. For this model, the authors proved that querying the edges of a low-stretch spanning tree of the input graph $G = (V,E)$ is sufficient to predict the remaining edge labels making a number of mistakes within a factor of order $(\log|V|)^2\log\log|V|$ from the theoretical optimum. The overall running time is $O(|E|\ln|V|)$. This result leaves two main problems open: First, low-stretch trees are a powerful structure, but the algorithm to construct them is not easy to implement. Second, the tree-based analysis of~\cite{cgvz12} does not generalize to query budgets larger than $|V|-1$ (the edge set size of a spanning tree). In this paper we introduce a different active learning approach for link classification that can accomodate a large spectrum of query budgets. We show that on \textsl{any} graph with $\Omega(|V|^{3/2})$ edges, a query budget of $\scO(|V|^{3/2})$ is sufficient to predict the remaining edge labels within a \textsl{constant} factor from the optimum. 
More in general, we show that a budget of at most order of $|V| + \bigl(\frac{|V|}{k}\bigr)^{3/2}$ queries is sufficient to make a number of mistakes within a factor of $\scO(k)$ from the optimum with a running time of order $|E| + (|V|/k)\log(|V|/k)$. Hence, a query budget of $\Theta(|V|)$, of the same order as the algorithm based on low-strech trees, achieves an optimality factor $\scO(|V|^{1/3})$ with a running time of just $\mathcal{O}(|E|)$.

At the end of the paper we also report on a preliminary set of experiments on medium-sized synthetic and
real-world datasets, where a simplified algorithm suggested by our theoretical findings 
is compared against the best performing spectral heuristics based on the same inductive bias. 
Our algorithm seems to perform similarly or better than these heuristics.

\section{Preliminaries and notation}
We consider undirected and connected
graphs $G = (V,E)$ with unknown edge labeling $Y_{i,j} \in \spin$ for each $(i,j) \in E$. 
Edge labels can collectively be represented by the associated {\em signed} adjacency matrix $Y$, 
where $Y_{i,j}=0$ whenever $(i,j) \not\in E$. In the sequel, the edge-labeled graph $G$ will be 
denoted by $(G,Y)$. 

We define a simple stochastic model for assigning binary labels $Y$ to the edges of $G$.
This is used as a basis and motivation for the design of our link classification strategies.
As we mentioned in the introduction, a good trade-off between accuracy and efficiency in link 
classification is achieved by assuming that the labeling is well approximated by a two-clustering of the nodes. 
Hence, our stochastic labeling
model assumes that edge labels are obtained by perturbing an underlying labeling which is initially consistent with an arbitrary (and unknown) two-clustering. More formally, given an undirected and connected graph $G = (V,E)$, 
the labels $Y_{i,j}\in\spin$, for $(i,j) \in E$, are assigned as follows. First, the nodes in $V$ are arbitrarily partitioned into two sets, and labels $Y_{i,j}$ are initially assigned consistently
with this partition (within-cluster edges are positive and between-cluster edges are negative). 
Note that the consistency is equivalent to the following {\em multiplicative rule}: 
For any $(i,j) \in E$, the label $Y_{i,j}$ is equal to the product of signs on the edges of 
\textsl{any} path connecting  $i$ to $j$ in $G$. This is in turn equivalent to say that
any simple cycle within the graph contains an {\em even} number of negative edges. 
Then, given a nonnegative constant $p < \tfrac{1}{2}$,
labels are randomly flipped in such a way that
$\Pr\bigl(\text{$Y_{i,j}$ is flipped}\bigr) \le p$ for each $(i,j) \in E$.
We call this a $p$-stochastic assignment. Note that this model allows for correlations 
between flipped labels.

A learning algorithm in the link classification setting receives a training set of signed edges and, 
out of this information, builds a prediction model for the labels of the remaining edges.
It is quite easy to prove a lower bound on the number of mistakes that any learning algorithm makes 
in this model.
\begin{fact}
\label{f:lower}
For any undirected graph $G = (V,E)$, any training set $E_0 \subset E$ of edges, and any learning algorithm 
that is given the labels of the edges in $E_0$, the number $M$ of mistakes made by $A$ on the remaining 
$E\setminus E_0$ edges satisfies 
$
	\E\,M \ge p\,\big|E\setminus E_0\big|
$,
where the expectation is with respect to a $p$-stochastic assignment of the labels $Y$.
\end{fact}
\begin{proof}
Let $Y$ be the following randomized labeling: first, edge labels are set consistently with an arbitrary two-clustering of $V$. Then, a set of $2p|E|$ edges is selected uniformly at random and the labels of these edges are set 
randomly (i.e., flipped or not flipped with equal probability). Clearly, 
$\Pr(\text{$Y_{i,j}$ is flipped}) = p$ for each $(i,j) \in E$. Hence this is a $p$-stochastic assignment 
of the labels. Moreover, $E\setminus E_0$ contains in expectation $2p\big|E\setminus E_0\big|$ 
randomly labeled edges, on which $A$ makes $p\big|E\setminus E_0\big|$ mistakes in expectation.
\end{proof}

In this paper we focus on active learning algorithms. An active learner for link 
classification first constructs a query set $E_0$  of edges, and then receives the labels of all edges in the query set. 
Based on this training information, the learner builds a prediction model for the labels of the remaining edges $E\setminus E_0$. 
We assume that the only labels ever revealed to the learner are those in the query set. In particular, no labels are 
revealed during the prediction phase.
It is clear from Fact~\ref{f:lower} that any active learning algorithm that queries the labels of at most
a constant fraction of the total number of edges will make on average $\Omega(p|E|)$ mistakes.

We often write $V_G$ and $E_G$ to denote, respectively, the node set and the edge set of some underlying graph $G$.
For any two nodes $i, j \in V_G$, $\path(i,j)$ is any path in $G$ having $i$ and $j$ as terminals, and $|\path(i,j)|$ is its length (number of edges). The diameter $D_G$ of a graph $G$ is the maximum over pairs $i,j \in V_G$ of the shortest path between $i$ and $j$.
Given a tree $T = (V_T,E_T)$ in $G$, and two nodes $i, j \in V_T$,
we denote by $d_T(i,j)$ the distance of $i$ and $j$ within $T$, i.e., the length of the
(unique) path $\path_T(i,j)$ connecting the two nodes in $T$. 
Moreover, $\pi_T(i,j)$ denotes the {\em parity}
of this path, i.e., the product of edge signs along it.
When $T$ is a rooted tree, we denote by $\chd_T(i)$ the set of children of $i$ in $T$. 
Finally, given two disjoint subtrees 
$T', T'' \subseteq G$ such that $V_{T'} \cap V_{T''} \equiv \emptyset$, we let
\(
E_G(T',T'') \equiv \bigl\{(i,j) \in E_G\, :\, i \in V_{T'},\, j \in V_{T''}\bigr\}~.
\)

\section{Algorithms and their analysis}\label{s:th}
In this section, we introduce and analyze a family of active learning algorithms for link classification. 
The analysis is carried out under the $p$-stochastic assumption.
As a warm up, we start off recalling the connection to the theory of low-stretch spanning trees 
(e.g., \cite{EEST10}), which turns out to be useful in the important special case when the active learner is afforded to query only $|V|-1$ labels.

Let $\Eflip \subset E$ denote the (random) subset of edges whose labels have been flipped in a $p$-stochastic assignment, and 
consider the following class of active learning algorithms parameterized by an arbitrary 
spanning tree $T = (V_T,E_T)$ of $G$. The algorithms in this class use $E_0 = E_T$ as query set. 
The label of any test edge $e' = (i,j)\not\in E_T$ is predicted as the parity $\pi_T(e')$. 
Clearly enough, if a test edge $e'$ is predicted wrongly, then either $e'\in \Eflip$ or $\path_T(e')$ 
contains at least one flipped edge. Hence, the number of mistakes $M_T$ made by our active learner on the 
set of test edges $E\setminus E_T$ can be deterministically bounded by
%
\begin{equation}\label{e:detbound}
M_T \le |\Eflip| + \sum_{e' \in E\setminus E_T}\sum_{e \in E} \Ind{e \in \path_T(e')} \Ind{e \in \Eflip}
\end{equation}
where $\Ind{\cdot}$ denotes the indicator of the Boolean predicate at argument.
A quantity which can be related to $M_T$ is the {\em average stretch} of a spanning tree $T$
which, for our purposes, reduces to

\begin{center}
\(
    \frac{1}{|E|}\left[ |V|-1 + \sum_{e' \in E\setminus E_T} \bigl|\path_T(e')\bigr| \right]~.
\)
\end{center}

A stunning result of~\cite{EEST10} shows that every connected, undirected and unweighted graph 
has a spanning tree with an average stretch of just $\mathcal{O}\bigl(\log^2|V|\log\log|V|\bigr)$. 
If our active learner uses a spanning tree with the same low stretch, then the following result holds.

\begin{theorem}[\cite{cgvz12}]\label{th:randomadv}
Let $(G,Y) = ((V,E),Y)$ be a labeled graph with $p$-stochastic assigned labels $Y$. 
If the active learner queries the edges of a spanning tree $T = (V_T,E_T)$ 
with average stretch $\mathcal{O}\bigl(\log^2|V|\log\log|V|\bigr)$, then
$
    \E\,M_T \le p|E| \times \mathcal{O}\bigl(\log^2|V|\log\log|V|\bigr)
$.
\end{theorem}
%
%
We call the quantity multiplying $p\,|E|$ in the upper bound the \textsl{optimality factor} of the algorithm. 
Recall that Fact~\ref{f:lower} implies that this factor cannot be smaller than a constant when the query set 
size is a constant fraction of $|E|$.

Although low-stretch trees can be constructed in time $\scO\bigl(|E|\ln|V|\bigr)$, the algorithms are 
fairly complicated (we are not aware of available implementations), and the constants hidden in 
the asymptotics can be high.
Another disadvantage is that we are forced to use a query set of small and fixed size $|V|-1$. 
In what follows we introduce algorithms that overcome both limitations.

\newcommand{\algt}{\textsc{algtree}}
\newcommand{\scT}{\mathcal{T}}
 
%
%
%
 

A key aspect in the analysis of prediction performance is 
the ability to select a query set so that each test edge creates a short circuit with a training path. This is quantified by 
$\sum_{e \in E} \Ind{e \in \path_T(e')}$ in~(\ref{e:detbound}).
We make this explicit as follows.
Given a test edge $(i,j)$ and a path $\path(i,j)$ whose edges are queried
edges, we say that we are predicting label $Y_{i,j}$ {\em using path} $\path(i,j)$ 
Since $(i,j)$ closes $\path(i,j)$ into a circuit, in this case we
also say that $(i,j)$ is predicted using the circuit.
%
%
%
%

\begin{fact}\label{th:p}
Let $(G,Y) = ((V,E),Y)$ be a labeled graph with $p$-stochastic assigned labels $Y$. 
Given query set $E_0 \subseteq E$, the number $M$ of mistakes made 
when predicting test edges $(i,j) \in E\setminus E_0$ using training paths $\path(i,j)$ 
whose length is uniformly bounded by $\ell$ satisfies
\(
\E M \le \ell\, p\,|E\setminus E_0|~.
\)
\end{fact}
\begin{proof}
We have the chain of inequalities
\begin{align*}
\E M 
&\le \!\!\!\sum_{(i,j) \in E\setminus E_0 } \!\!\bigl( 1-(1-p)^{|\path(i,j)|} \bigr) \\
&\le \!\!\!\sum_{(i,j) \in E\setminus E_0 } \!\!\bigl( 1-(1-p)^{\ell} \bigr)\\
&\le \!\!\!\sum_{(i,j) \in E\setminus E_0 } \!\!\ell\,p\\ 
&\le \ell\,p\,|E\setminus E_0|~.
\end{align*}
\end{proof}

For instance, if the input graph $G = (V,E)$ has diameter $D_G$ and the queried
edges are those of a breadth-first spanning tree, which can be generated
in $O(|E|)$ time, then the above fact holds with $|E_0| = |V|-1$, and $\ell = 2\,D_G$.
Comparing to Fact \ref{f:lower} shows that this simple breadth-first
strategy is optimal up to constants factors whenever $G$ has a constant diameter.
This simple observation is especially relevant in the light of the typical
graph topologies encountered in practice, whose diameters are often small.
%
This argument is at the basis of our experimental comparison ---see Section~\ref{s:exp}~.

%


Yet, this mistake bound can be vacuous on 
graph having a larger diameter. Hence, one may think
of adding to the training spanning tree new edges 
so as to reduce the length of the circuits used for 
prediction, at the cost of increasing the size of the query set.
A similar technique based on short circuits has been used in~\cite{cgvz12}, 
the goal there being to solve the link classification
problem in a harder adversarial environment.
The precise tradeoff between prediction accuracy (as measured by the expected
number of mistakes) and fraction of queried edges is the main theoretical
concern of this paper. 

We now introduce an intermediate (and simpler) algorithm, called \alg, which improves on the optimality factor when the diameter $D_G$ is not small. In particular, 
we demonstrate that \alg\ achieves a good upper bound on the number of mistakes  
on any graph such that $|E| \ge 3|V|+\sqrt{|V|}$. This algorithm 
is especially effective when the input graph is dense, with an optimality factor 
between $\scO(1)$ and $\scO(\sqrt{|V|})$.
%
%
%
Moreover, the total time for predicting the test edges
scales linearly with the number of such edges, i.e.,
\alg\ predicts edges in {\em constant amortized} time.
Also, the space is linear in the size of the input graph.

The algorithm (pseudocode given in Figure~\ref{f:alg}) is parametrized by a positive integer $k$ ranging from 2 to $|V|$. 
The actual setting of $k$ depends on the graph topology and
the desired fraction of query set edges, and 
plays a crucial role in determining the prediction performance.
Setting $k \le D_G$ makes \alg\ reduce 
to querying only the edges of a breadth-first spanning tree of $G$, otherwise
it operates in a more involved way by splitting $G$ into smaller node-disjoint subtrees.
%

In a preliminary step (Line 1 in Figure~\ref{f:alg}), 
\alg\ draws an arbitrary breadth-first spanning tree $T = (V_T,E_T)$. 
Then subroutine $\tp(T,k)$ is used in a do-while loop to
split $T$ into vertex-disjoint subtrees $T'$
whose height is $k$ (one of them might have a smaller height). 
$\tp(T,k)$ is a very simple procedure that
performs a depth-first visit of the tree $T$ at argument. 
During this visit, each internal node may be visited several times
(during backtracking steps). We assign each node $i$ a tag $h_T(i)$ representing
the height of the subtree of $T$ rooted at $i$. $h_T(i)$ can be recursively computed
during the visit.
After this assignment, if we have $h_T(i)=k$ (or $i$ is the root of $T$) 
we return the subtree $T_i$ of $T$ rooted at $i$. 
Then \alg\ removes (Line~6) $T_i$ from $T$ along with all edges of $E_T$ which
are incident to nodes of $T_i$, and then iterates until $V_T$ gets empty.
By construction, the diameter of the generated subtrees will not be larger
than $2k$.
Let $\scT$ denote the set of these subtrees. For each $T' \in \scT$,
the algorithm queries all the labels of $E_{T'}$,
each edge $(i,j) \in E_G \setminus E_{T'}$ such that $i,j \in V_{T'}$
is set to be a test edge, and label $Y_{i,j}$ is predicted using $\path_{T'}(i,j)$
(note that this coincides with $\path_{T'}(i,j)$, since $T' \subseteq T$), 
that is, $\hat{Y}_{i,j} = \pi_T(i,j)$.
Finally, for each pair of distinct subtrees $T', T'' \in \scT$ such that
there exists a node of $V_{T'}$ adjacent to a node of $V_{T''}$, i.e.,
such that $E_G(T',T'')$ is not empty, we query the label of an arbitrarily
selected edge $(i',i'') \in E_G(T',T'')$ (Lines $8$ and $9$ in Figure~\ref{f:alg}).
Each edge $(u,v) \in E_G(T',T'')$ whose label has not been previously
queried is then part of the test set, and its label will be predicted as
$\hat{Y}_{u,v} \leftarrow \pi_T(u,i')\cdot Y_{i',i''}\cdot \pi_T(i'',v)$ 
(Line $11$). That is, using the path obtained by concatenating $\path_{T'}(u,i')$ 
to edge $(i',i'')$ to $\path_{T'}(i'',v)$. 

\begin{figure}[h!]
\hrule\vspace{.03in}
\begin{tabbing}
\hspace{.25in} \=\hspace{.10in} \= \hspace{.10in} \=  \hspace{.10in} \=  \hspace{.10in} \=  \hspace{.10in}\= \hspace{.25in} \=\hspace{.10in} \= \hspace{.0in} \=  \hspace{.0in} \=  \kill
\alg$(k)$ \qquad  Parameter: $k \ge 2$\\
Initialization: $\scT \leftarrow \emptyset$.\\
{\tt 1.\ }  Draw an arbitrary breadth-first spanning tree $T$ of $G$\\
{\tt 2.\ }  \textbf{Do} \\
{\tt 3.\ } \qquad $T' \leftarrow \tp(T,k)$, and query all labels in $E_{T'}$\\
{\tt 4.\ } \qquad $\scT \leftarrow \scT \cup \{T'\}$\\
{\tt 5.\ }  \qquad \textbf{For each} $i,j \in V_{T'}$, set predict\ $\hat{Y}_{i,j} \leftarrow \pi_T(i,j)$\\
{\tt 6.\ } \qquad $T \leftarrow T \setminus T'$\\
{\tt 7.\ } \textbf{While} ($V_{T} \not\equiv \emptyset$)\\
{\tt 8.\ }  \textbf{For each} $T',T'' \in \scT : T'\not\equiv T''$\\
{\tt 9.\ }  \qquad \textbf{If} $E_G(T',T'') \not\equiv \emptyset$ query the label of an arbitrary edge $(i',i'') \in E_G(T',T'')$\\
{\tt 10.}  \qquad \textbf{For each} $(u,v) \in E_G(T',T'') \setminus \{(i',i'')\}$,   
         with $i', u \in V_{T'}$ and $v,i'' \in V_{T''}$\\
{\tt 11.}  \qquad\qquad predict    $\hat{Y}_{u,v} \leftarrow \pi_{T'}(u,i')\cdot Y_{i',i''}\cdot \pi_{T''}(i'',v)$\\
\end{tabbing}
\vspace{-0.14in}\hrule 
\caption{\label{f:alg}\alg\ pseudocode.} 
\end{figure}
%
\begin{figure}[h!]
\hrule\vspace{.03in}
\begin{tabbing}
\hspace{.25in} \=\hspace{.10in} \= \hspace{.10in} \=  \hspace{.10in} \=  \hspace{.10in} \=  \hspace{.10in}\= \hspace{.25in} \=\hspace{.10in} \= \hspace{.10in} \=  \hspace{.10in} \=  \kill
$\tp(T,k)$ \qquad  Parameters: tree $T$, $k \ge 2$. \\
{\tt 1.\ }  Perform a depth-first visit of $T$ starting from the root. \\
{\tt 2.\ }  \textbf{During the visit} \\
{\tt 3.\ }  \qquad \textbf{For each} $i \in V_{T}$ visited for the $|1+\chd_T(i)|$-th time 
             (i.e., the last visit of $i)$\\
{\tt 4.\ } \qquad \qquad \textbf{If} $i$ is a leaf set $h_T(i) \leftarrow 0$\\
{\tt 5.\ } \qquad  \qquad\textbf{Else} set $h_T(i) \leftarrow 1+\max\{h_T(j) : j \in \chd_T(i)\}$\\
{\tt 6.\ } \qquad  \qquad\textbf{If} $h_T(i)=k$ or $i \equiv T$'s root \textbf{return} subtree rooted at $i$\\
\end{tabbing}
\vspace{-0.14in}\hrule 
\caption{\label{f:tp}\tp\ pseudocode.} 
\end{figure}
%

The following theorem\footnote
{
Due to space limitations long proofs are presented in the supplementary material.
}
quantifies the number of mistakes made by 
\alg. The requirement on the graph density in the statement, i.e., 
$|V| - 1 + \frac{|V|^2}{2k^2}+\frac{|V|}{2k} \leq \frac{|E|}{2}$ 
implies that the test set is not larger than the query set.
This is a plausible assumption in active learning scenarios,
and a way of adding meaning to the bounds.

%
%
%

\begin{theorem}\label{t:alg}
For any integer $k \ge 2$, the number $M$ of mistakes made by \alg\ 
on any graph $G(V,E)$ with $|E| \ge 2|V|-2 + \frac{|V|^2}{k^2} + \frac{|V|}{k}$ satisfies
$
\E M \le \min\{4k+1, 2D_G\} p|E|
$,
while the query set size 
is bounded by $|V|-1 + \frac{|V|^2}{2k^2} + \frac{|V|}{2k} \le \frac{|E|}{2}$. 
\end{theorem}

\newcommand{\algs}{\textsc{starMaker}}
\newcommand{\es}{\textsc{extractStar}}
\newcommand{\algst}{\textsc{treeletStar}}

\subsection{Refinements}\label{ss:refine}
We now refine the simple argument leading to \alg, and present our active link classifier.
The pseudocode of our refined algorithm, called \algs, follows that of Figure~\ref{f:alg} with the following differences: Line~1 is dropped (i.e., \algs\ does not draw an initial spanning tree), and the call to \tp\ in Line~3 is replaced by a call to \es. This new subroutine just selects the star $T'$ centered on the node of $G$ having largest degree, and queries all labels of the edges in $E_{T'}$. The next 
result shows that this algorithm gets 
a \textsl{constant} optimality factor while using a query set of size $\scO(|V|^{3/2})$.

\begin{theorem}\label{th:algs}
The number $M$ of mistakes made by \algs\ on any given graph $G(V,E)$ with $|E| \ge 2|V|-2 + 2|V|^{\frac{3}{2}}$ satisfies
$
\E M \le 5\,p|E|
$,
while the query set size is upper bounded by $|V|-1 + |V|^{\frac{3}{2}} \le \frac{|E|}{2}$. 
\end{theorem}

Finally, we combine \algs\ with \alg\ so as to obtain an algorithm, called \algst, that can work
with query sets smaller than $|V|-1 + |V|^{\frac{3}{2}}$ labels.
\algst\ is parameterized by an integer $k$ and follows Lines 1--6 of Figure~\ref{f:alg} creating a set $\scT$ of trees through repeated calls to \tp. Lines 7--11 are instead replaced by the following procedure:
a graph $G' = (V_{G'},E_{G'})$ is created such that: (1) each node in $V_{G'}$ corresponds to a tree in $\scT$, (2) there exists an edge in $E_{G'}$ if and only if the two corresponding trees of $\scT$ are connected 
by at least one edge of $E_G$. 
Then, \es\ is used to generate a set $\scS$ of stars of vertices of $G'$, i.e., stars of trees of $\scT$. 
Finally, for each pair of distinct stars $S', S'' \in \scS$ connected by at least one edge in $E_{G}$,
the label of an arbitrary edge in $E_G(S',S'')$ is queried. The remaining edges are all predicted.

\begin{theorem}\label{t:algst}
For any integer $k \geq 2$ and for any graph $G = (V,E)$ with $|E| \ge 2|V|-2 + 2\bigl(\frac{|V|-1}{k}+1\bigr)^{\frac{3}{2}}$, the number $M$ of mistakes made by $\algst(k)$ on $G$ satisfies
$
\E M = \scO(\min\{k, D_G\})\, p|E|
$,
while the query set size is bounded by $|V|-1 + \bigl(\frac{|V|-1}{k}+1\bigr)^{\frac{3}{2}} \le \frac{|E|}{2}$. 
\end{theorem}

Hence, even if $D_G$ is large, setting $k = |V|^{1/3}$ yields a $\scO(|V|^{1/3})$ optimality 
factor
just by querying $\scO(|V|)$ edges. On the other hand, a truly constant optimality factor
is obtained by querying as few as $\scO(|V|^{3/2})$ edges (provided the graph 
has sufficiently many edges).
As a direct consequence (and surprisingly enough), on graphs which are only moderately dense
we need not observe too many edges in order to achieve a constant optimality factor. 
It is instructive to compare the bounds obtained by \algst\ to the ones
we can achieve by using the \algb\ algorithm of \cite{cgvz12}, or the low-stretch spanning trees
given in Theorem \ref{th:randomadv}.

Because \algb\ operates within a harder adversarial setting, it is easy to show that 
Theorem 9 in \cite{cgvz12} extends to the $p$-stochastic assignment model 
by replacing $\Delta_2(Y)$ with $p|E|$ therein.\footnote
{
This theoretical comparison is admittedly unfair, as \algb\ has been designed to work 
in a harder setting than $p$-stochastic. Unfortunately, we are not aware of any other 
general active learning scheme for link classification to compare with.
}
The resulting optimality factor is of order $\bigl(\frac{1-\alpha}{\alpha}\bigr)^{\frac{3}{2}} \sqrt{|V|}$, where $\alpha \in (0,1]$ is the fraction of queried edges out of the total number of edges.
%
%
A quick comparison to Theorem \ref{t:algst} reveals that \algst\ achieves
a sharper mistake bound for any value of $\alpha$. For instance, in order to obtain
an optimality factor which is lower than $\sqrt{|V|}$, \algb\ has to query 
in the worst case a fraction of edges that goes to one as $|V|\to\infty$. 
On top of this, 
%
%
our algorithms are faster and easier to implement ---see Section \ref{ss:compl}.
%

Next, we compare to query sets produced by low-stretch spanning trees.
%
%
A low-stretch spanning tree achieves a polylogarithmic optimality factor by querying $|V|-1$ edge labels. 
The results in~\cite{EEST10} show that we cannot hope to get a better optimality factor 
using a single low-stretch spanning tree combined by the analysis in~(\ref{e:detbound}).
For a comparable amount $\Theta(|V|)$ of queried labels, Theorem \ref{t:algst} offers the larger optimality factor $|V|^{1/3}$. However, we can get a \textsl{constant} optimality factor by increasing the query set size to $\scO(|V|^{3/2})$. It is not clear how multiple low-stretch trees could be combined to get a similar scaling.



\newcommand{\rt}{\mathrm{root}}

\subsection{Complexity analysis and implementation}\label{ss:compl}
We now compute bounds on time and space requirements for our three algorithms. Recall the different lower bound conditions on the graph density that must hold to ensure that the query set size is not larger than the test set size. 
These were $|E| \ge 2|V|-2 + \frac{|V|^2}{k^2} + \frac{|V|}{k}$ for \alg$(k)$ in Theorem~\ref{t:alg}, $|E| \ge 2|V|-2 + 2|V|^{\frac{3}{2}}$ for \algs\ in Theorem~\ref{th:algs}, and
$|E| \ge 2|V|-2 + 2\Bigl(\frac{|V|-1}{k}+1\Bigr)^{\frac{3}{2}}$ for \algst$(k)$ in Theorem \ref{t:algst}.
%
%
\begin{theorem}\label{t:compl}
For any input graph $G = (V,E)$ which is dense enough to
ensure that the query set size is no larger than the test set size, 
the total time needed for predicting all test labels is:
 
\begin{align*}
\scO(|E|) &\qquad \text{for \alg$(k)$ and for all $k$}
\\
\scO\bigl(|E| + |V| \log |V|\bigr) &\qquad \text{for \algs}
\\
\scO\left(|E| + \frac{|V|}{k} \log\frac{|V|}{k}\right) &\qquad \text{for \algst$(k)$ and for all $k$.}
\end{align*}
In particular, whenever $k|E| = \Omega(|V|\log |V|)$ we have that \algst$(k)$ works in constant amortized time.
For all three algorithms, the space required is always linear in the input graph size $|E|$.
\end{theorem}

\section{Experiments}\label{s:exp}
 

%

\newcommand{\tY}{Y_{\mathrm{train}}}


In this preliminary set of experiments we only tested the predictive performance of $\alg(|V|)$. This corresponds to querying only the edges of the initial spanning tree $T$ and predicting all remaining edges $(i,j)$ via the parity of $\path_T(i,j)$. The spanning tree $T$ used by $\alg$ is a shortest-path spanning tree generated by a breadth-first visit of the graph (assuming all edges have unit length). As the choice of the starting node in the visit is arbitrary, we picked the highest degree node in the graph. Finally, we run through the adiacency list of each node in random order, which we empirically observed to improve performance.

Our baseline is the heuristic ASymExp from~\cite{KLB09} which, among the many spectral heuristics proposed there, 
turned out to perform best on all our datasets. With integer parameter $z$, ASymExp$(z)$ predicts using a spectral transformation of the training sign matrix $\tY$, whose only non-zero entries are the signs of the training edges. The label of edge $(i,j)$ is predicted using $\bigl(\exp(\tY(z))\bigr)_{i,j}$. Here $\exp\bigl(\tY(z) \bigr) = U_z \exp(D_z) U_z^{\top}$, where $U_z D_z U_z^{\top}$ is the spectral decomposition of $\tY$ containing only the $z$ largest eigenvalues and their corresponding eigenvectors. Following~\cite{KLB09}, 
we ran ASymExp$(z)$ with the values $z = 1, 5, 10, 15$. This heuristic uses the two-clustering bias as follows : 
expand $\exp(\tY)$ in a series of powers $\tY^n$. Then each $\bigl(\tY^n)_{i,j}$ is a sum of values of paths 
of length $n$ between $i$ and $j$. 
Each path has value $0$ if it contains at least one test edge, otherwise its value 
equals the product of queried labels on the path edges. Hence, the sign of $\exp(\tY)$ is the sign of a linear combination of path values, each corresponding to a prediction consistent with the two-clustering bias ---compare this to the multiplicative rule used by \alg. Note that ASymExp and the other spectral heuristics from~\cite{KLB09} have all running times of order $\Omega\bigl(|V|^2\bigr)$.

We performed a first set of experiments on synthetic signed graphs created from a subset of the 
USPS digit recognition dataset. We randomly selected 500 examples labeled ``1'' and 500 
examples labeled ``7'' (these two classes are not straightforward to tell apart).
Then, we created a graph using a $k$-NN rule with $k=100$. The edges were labeled as follows: all edges incident to nodes with the same USPS label were labeled $+1$; 
all edges incident to nodes with different USPS labels were labeled $-1$.
Finally, we randomly pruned the positive edges so to achieve an unbalance of about $20\%$ between the two
classes.\footnote
{
This is similar to the class unbalance of real-world signed networks ---see below.
}
Starting from this edge label assignment, which is consistent with the two-clustering associated with the USPS labels, we generated a $p$-stochastic label assignment by flipping the labels of a random subset of the edges. Specifically, we used the three following synthetic datasets:

\textbf{DELTA0:} No flippings ($p = 0$), $1,\!000$ nodes and $9,\!138$ edges; 

\textbf{DELTA100:} 100 randomly chosen labels of DELTA0 are flipped; 

\textbf{DELTA250:} 250 randomly chosen labels of DELTA0 are flipped. 

We also used three real-world datasets:

\textbf{MOVIELENS:} A signed graph we created using Movielens ratings.\footnote
{
\texttt{www.grouplens.org/system/files/ml-1m.zip}.
}
We first normalized the ratings by subtracting from each user rating the average rating of that user. Then, we created a user-user matrix of cosine distance similarities. This matrix was sparsified by zeroing each entry smaller than $0.1$ and removing all self-loops. Finally, we took the sign of each non-zero entry. The resulting graph has $6,\!040$ nodes and $824,\!818$ edges 
($12.6\%$ of which are negative).

\textbf{SLASHDOT:} The biggest strongly connected component of a snapshot of the Slashdot social network,\footnote
{
\texttt{snap.stanford.edu/data/soc-sign-Slashdot081106.html}.
}
similar to the one used in~\cite{KLB09}. 
This graph has $26,\!996$ nodes and $290,\!509$ edges ($24.7\%$ of which are negative).

\textbf{EPINIONS:} The biggest strongly connected component of a snapshot of the Epinions signed network,\footnote
{
\texttt{snap.stanford.edu/data/soc-sign-epinions.html}.
}
similar to the one used in~\cite{LHK10,MA06}. 
This graph has $41,\!441$ nodes and $565,\!900$ edges ($26.2\%$ of which are negative).

Slashdot and Epinions are originally directed graphs. 
We removed the reciprocal edges with mismatching labels (which turned out to be only a few), 
and considered the remaining edges as undirected.

The following table summarizes the key statistics of each dataset: Neg.\ is the fraction of
negative edges, $|V|/|E|$ is the fraction of edges queried by \alg($|V|$), 
and Avgdeg is the average degree of the nodes of the network.

\begin{center}
\begin{tabular}{l|r|r|r|r|r}
Dataset   &$|V|$ &$|E|$  &Neg.    &$|V|/|E|$       &Avgdeg \\
\hline
DELTA0    &1000  &9138   &21.9\%  &10.9\%          &18.2\\
DELTA100  &1000  &9138   &22.7\%  &10.9\%          &18.2\\
DELTA250  &1000  &9138   &23.5\%  &10.9\%          &18.2\\
\hline
SLASHDOT  &26996 &290509 &24.7\%  &9.2\%           &21.6\\
EPINIONS  &41441 &565900 &26.2\%  &7.3\%           &27.4\\
MOVIELENS &6040  &824818 &12.6\%  &0.7\%           &273.2
\end{tabular}
\end{center}
\begin{figure*}
\begin{center}
\begin{tabular}{c c c}
\includegraphics[width=54mm]{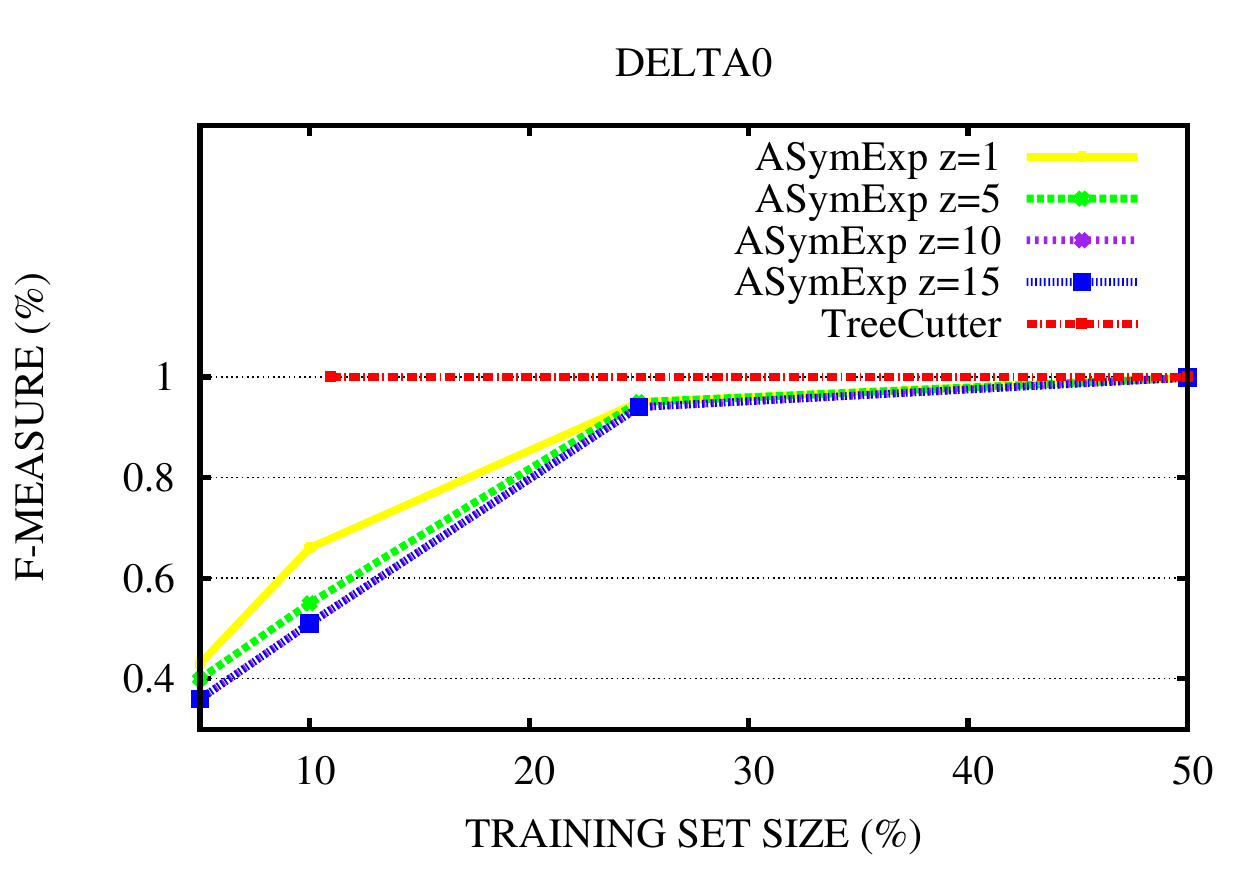}&
\includegraphics[width=54mm]{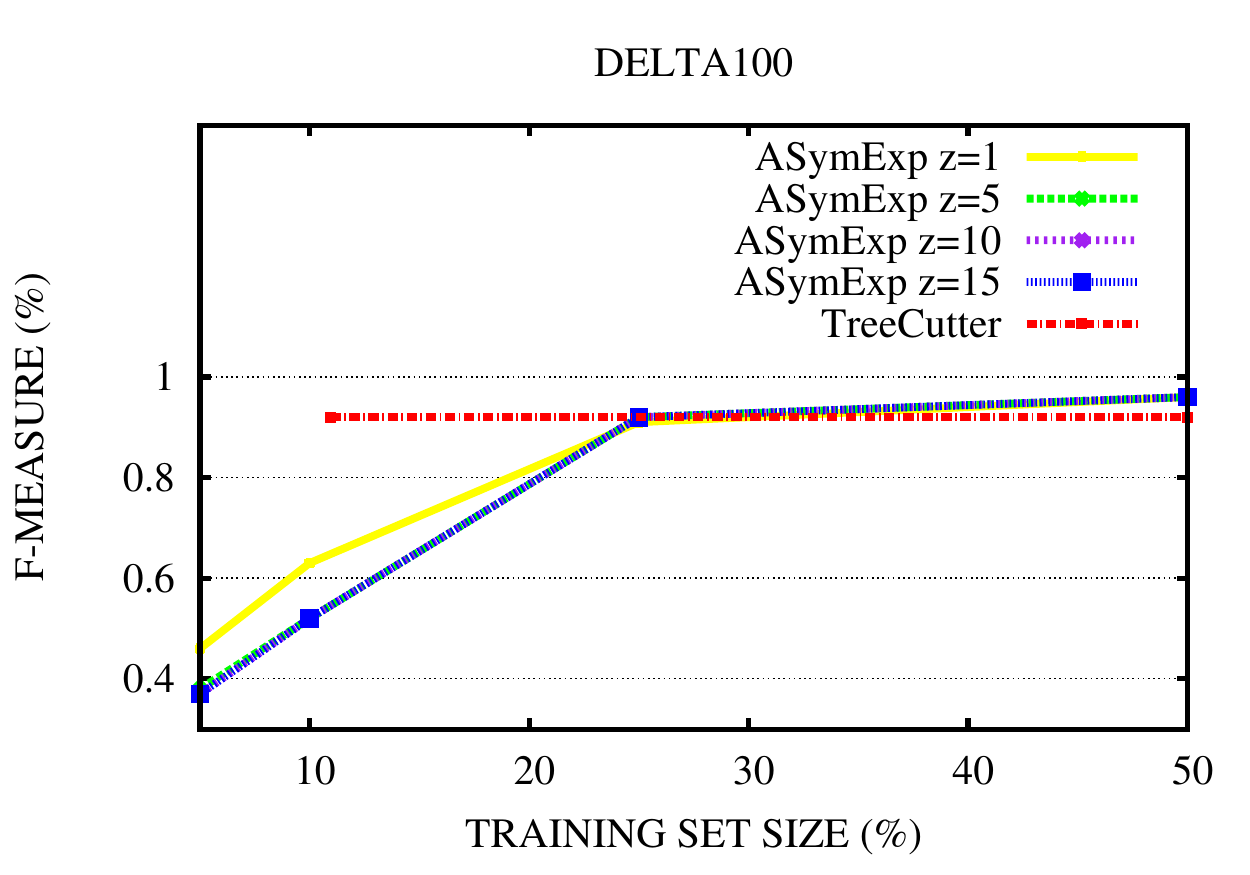}&
\includegraphics[width=54mm]{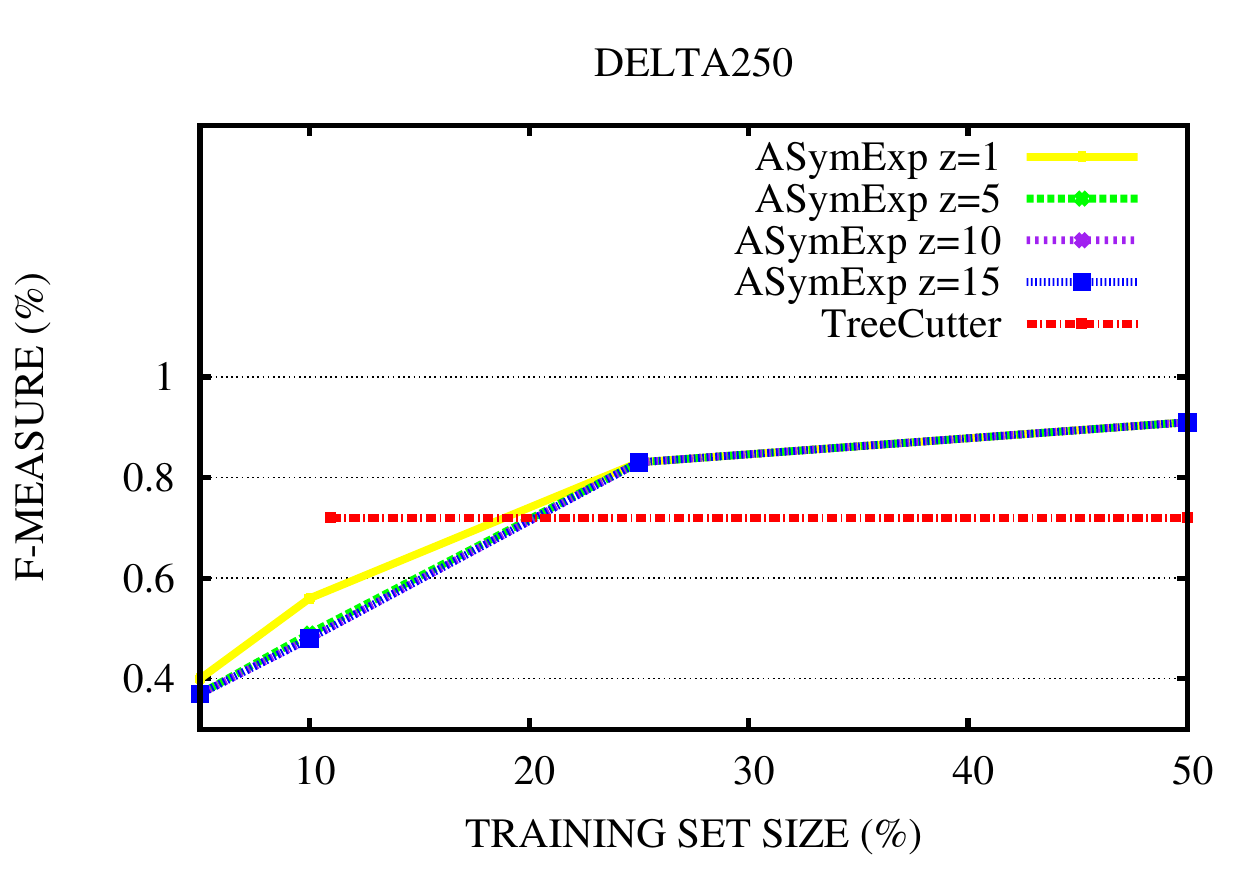}\\
\includegraphics[width=54mm]{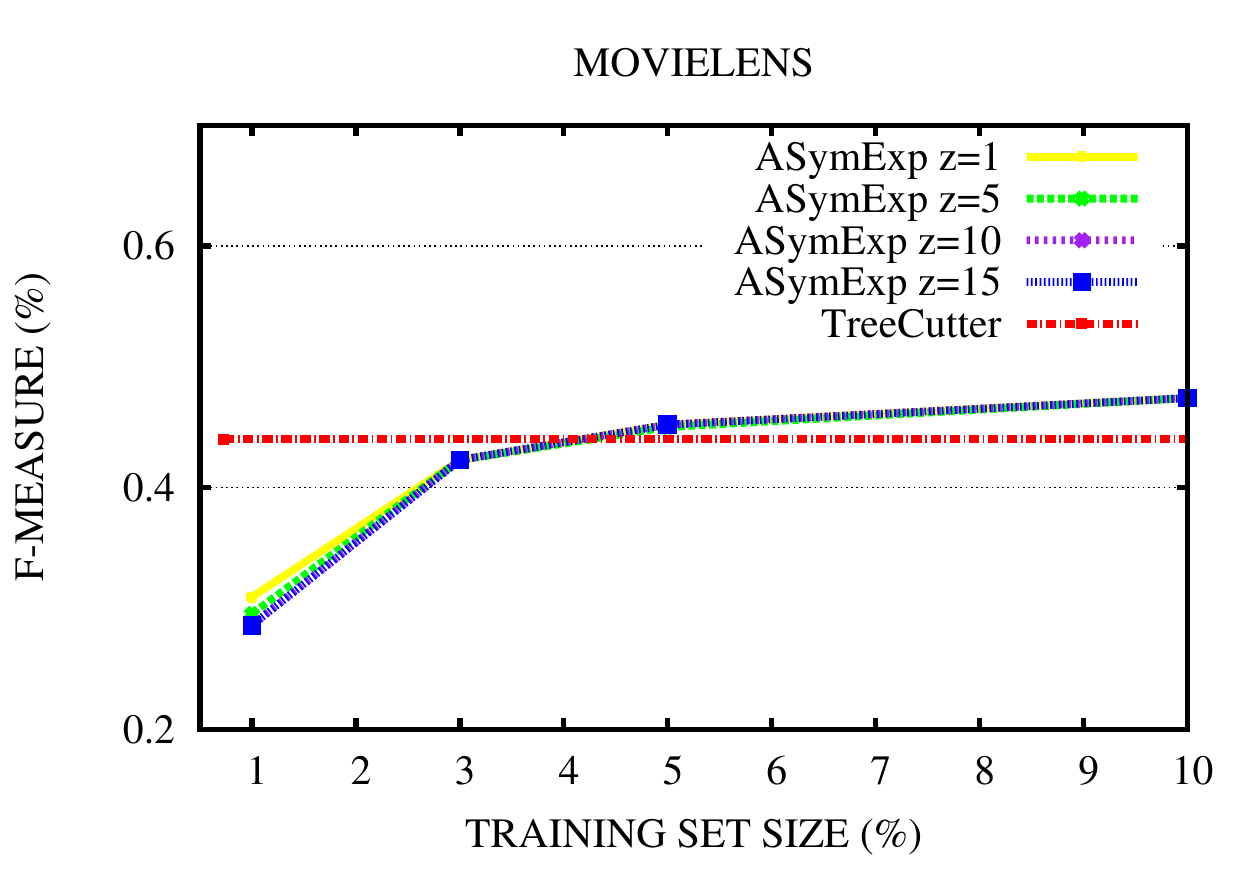}&
\includegraphics[width=54mm]{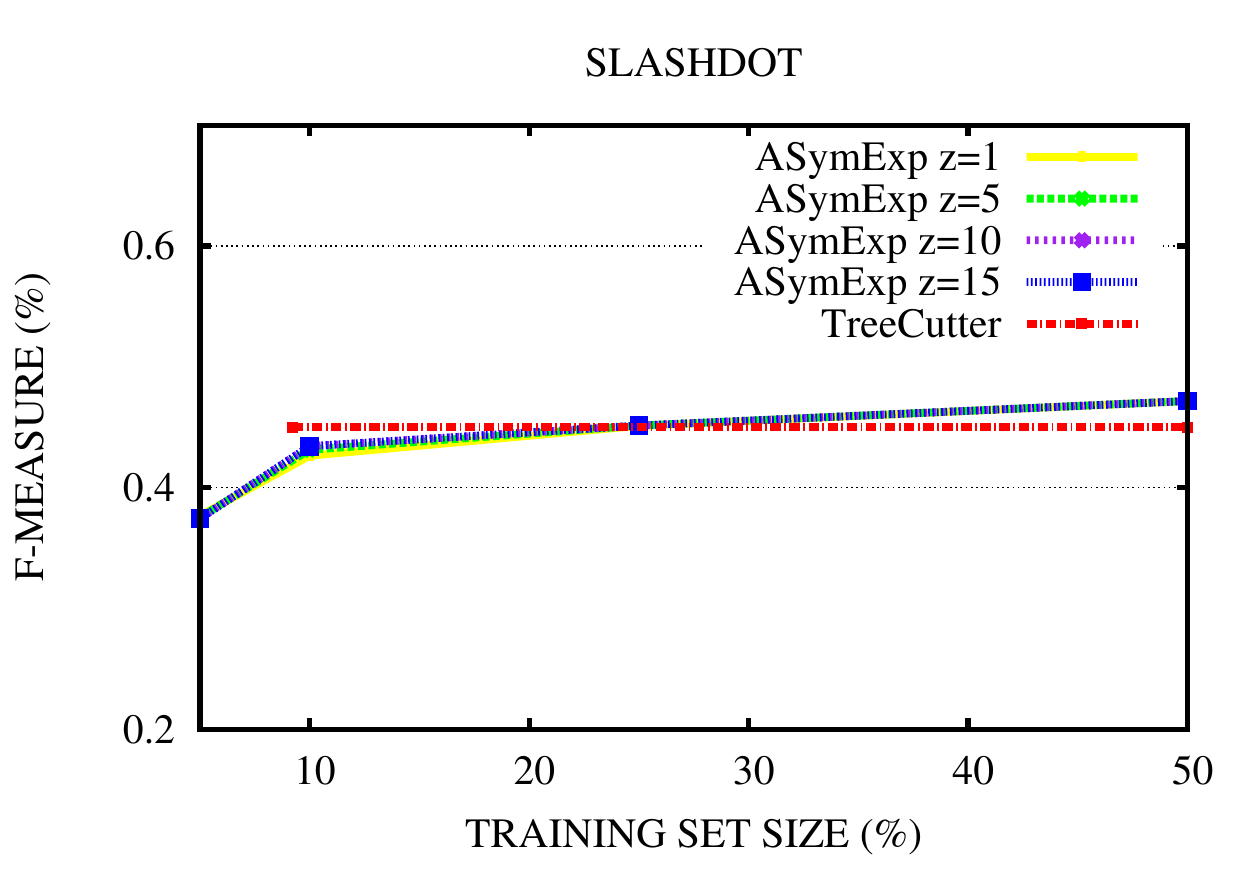}&
\includegraphics[width=54mm]{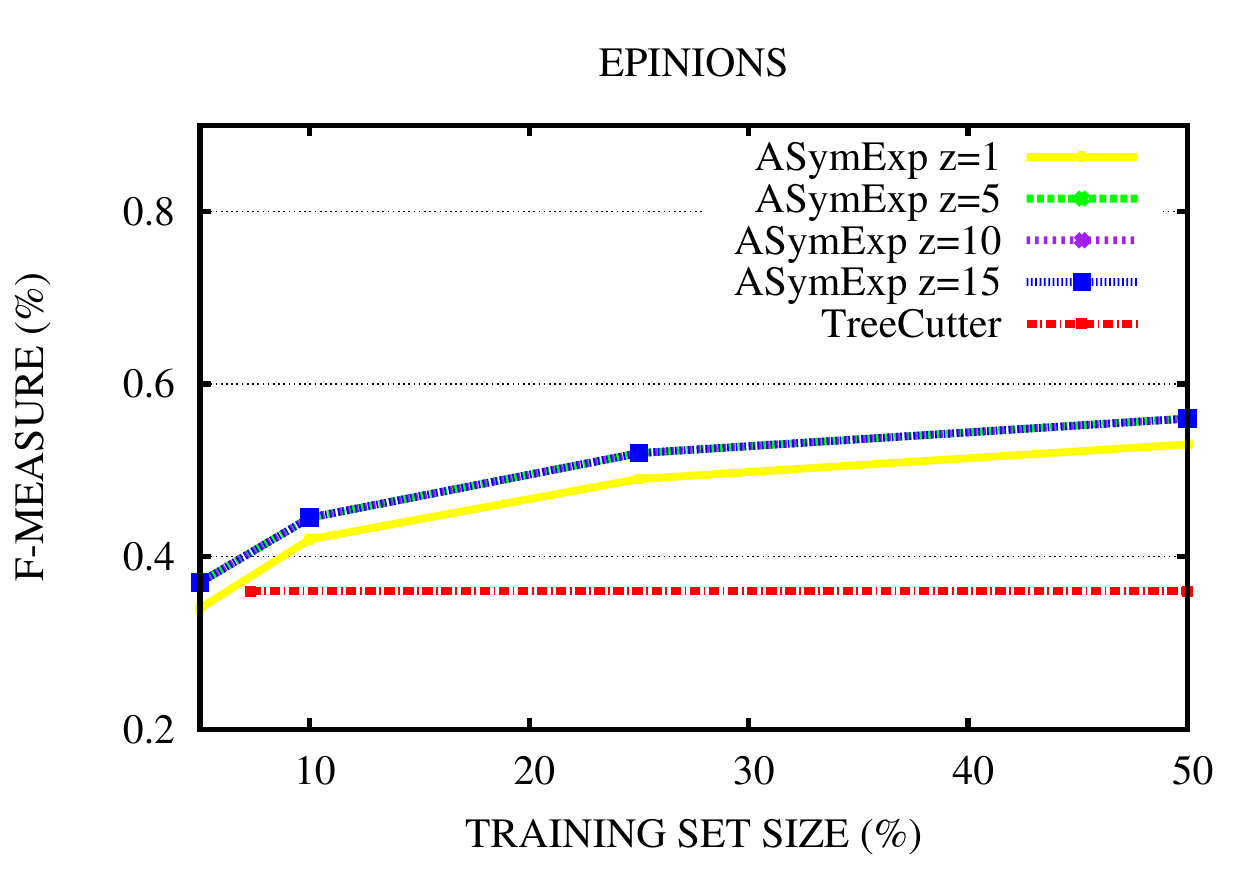}
\end{tabular}
\end{center}
\caption{\label{f:1}
F-measure against training set size for $\alg(|V|)$ and ASymExp$(z)$ with different values of $z$ on both synthetic and real-world datasets.
By construction, $\alg$ never makes a mistake when the labeling is consistent with a two-clustering. So on DELTA0 \alg\ does not make mistakes whenever the training set contains at least one spanning tree. With the exception of EPINIONS, $\alg$ outperforms ASymExp using a much smaller training set. We conjecture that ASymExp responds to the bias not as well as \alg, which on the other hand is less robust than ASymExp to bias violations (supposedly, the labeling of EPINIONS).
}
\end{figure*}




 
Our results are summarized in Figure \ref{f:1}, where we plot F-measure (preferable to accuracy due to the class unbalance) against the fraction of training (or query) set size. On all datasets, but MOVIELENS, the training set size for ASymExp ranges across the values 5\%, 10\%, 25\%, and 50\%. Since MOVIELENS has a higher density, we decided to reduce those fractions to 1\%, 3\%, 5\% and 10\%. $\alg(|V|)$ uses a single spanning tree, and thus we only have a 
single query set size value. 
All results are averaged over ten runs of the algorithms. The randomness in ASymExp is due to the random draw of the training set. The randomness in \alg($|V|$) is caused by the randomized breadth-first visit.





\section{Conclusions and work in progress}
We have built on the recent work~\cite{cgvz12}, so as to generalize the results contained therein
to query budgets larger than $|V|-1$ (the edge set size of a spanning tree). We also provided 
algorithms which are easier to implement than low-stretch spanning trees. A research avenue we
are currently exploring is whether we can combine the edge information with information
possibly contained in the nodes of a network. The suite of papers \cite{cgv09,cgvz10a,cgvz10b,cgv11,cgvz11} 
is a good starting for this investigation.

\bibliographystyle{plain}

\begin{thebibliography}{15}
 
\bibitem {ch56}
Cartwright, D. and Harary, F.
\newblock Structure balance: A generalization of {H}eider's theory.
\newblock \emph{Psychological review}, 63\penalty0 (5):\penalty0 277--293,
  1956.

\bibitem{cgv09}
Cesa-Bianchi, N., Gentile, C., Vitale, F.
Fast and optimal prediction of a labeled tree.
Proc. of of the 22nd Conference on Learning Theory (COLT 2009). 

\bibitem{cgv11}
Cesa-Bianchi, N., Gentile, C., Vitale, F.
Predicting the labels of an unknown graph via adaptive exploration.
Theoretical Computer Science , special issue on Algorithmic Learning Theory, 412/19 (2011), 
pp. 1791--1804. 

\bibitem{cgvz10a}
Cesa-Bianchi, N., Gentile, C., Vitale, F., Zappella, G.
Active learning on trees and graphs.
Proc. of the 23rd Conference on Learning Theory (COLT 2010).

\bibitem{cgvz10b}
Cesa-Bianchi, N., Gentile, C., Vitale, F., Zappella, G.
Random spanning trees and the prediction of weighted graphs.
Proc. of the 27th International Conference on Machine Learning (ICML 2010). 

\bibitem{cgvz11}
Cesa-Bianchi, N., Gentile, C., Vitale, F., Zappella, G.
See the tree through the lines: the Shazoo algorithm
Proc. of the 25th conference on Neural Information processing Systems (NIPS 2011). 

\bibitem{cgvz12}
Cesa-Bianchi, N., Gentile, C., Vitale, F., Zappella, G.
\newblock A correlation clustering approach to link classification in signed
  networks.
\newblock In \emph{Proceedings of the 25th conference on learning theory (COLT
  2012)}.

\bibitem{cntd11}
Chiang, K., Natarajan, N., Tewari, A., and Dhillon, I.
\newblock Exploiting longer cycles for link prediction in signed networks.
\newblock In \emph{Proceedings of the 20th ACM Conference on Information and
  Knowledge Management (CIKM)}. ACM, 2011.

\bibitem {EEST10}
Elkin, M., Emek, Y., Spielman, D.A., and Teng, S.-H.
\newblock Lower-stretch spanning trees.
\newblock \emph{SIAM Journal on Computing}, 38\penalty0 (2):\penalty0 608--628,
  2010.

\bibitem{fia11}
Facchetti, G., Iacono, G., and Altafini, C.
\newblock Computing global structural balance in large-scale signed social
  networks.
\newblock \emph{PNAS}, 2011.

\bibitem{GG06}
Giotis, I. and Guruswami, V.
\newblock {Correlation clustering with a fixed number of clusters}.
\newblock In \emph{Proceedings of the Seventeenth Annual ACM-SIAM Symposium on
  Discrete Algorithms}, pp.\  1167--1176. ACM, 2006.

\bibitem {GKRT04}
Guha, R., Kumar, R., Raghavan, P., and Tomkins, A.
\newblock {Propagation of trust and distrust}.
\newblock In \emph{Proceedings of the 13th international conference on World
  Wide Web}, pp.\  403--412. ACM, 2004.

\bibitem {ha53}
Harary, F.
\newblock On the notion of balance of a signed graph.
\newblock \emph{Michigan Mathematical Journal}, 2\penalty0 (2):\penalty0
  143--146, 1953.

\bibitem {hei46}
Heider, F.
\newblock Attitude and cognitive organization.
\newblock \emph{J. Psychol}, 21:\penalty0 107--122, 1946.

\bibitem {Hou05}
Hou, Y.P.
\newblock {Bounds for the least Laplacian eigenvalue of a signed graph}.
\newblock \emph{Acta Mathematica Sinica}, 21\penalty0 (4):\penalty0 955--960,
  2005.

\bibitem{KLB09}
Kunegis, J., Lommatzsch, A., and Bauckhage, C.
\newblock {The Slashdot Zoo: Mining a social network with negative edges}.
\newblock In \emph{Proceedings of the 18th International Conference on World
  Wide Web}, pp.\  741--750. ACM, 2009.

\bibitem {MA06}
Leskovec, J., Huttenlocher, D., and Kleinberg, J.
\newblock {Trust-aware bootstrapping of recommender systems}.
\newblock In \emph{Proceedings of ECAI 2006 Workshop on Recommender Systems},
  pp.\  29--33. ECAI, 2006.

\bibitem {LHK10}
Leskovec, J., Huttenlocher, D., and Kleinberg, J.
\newblock {Signed networks in social media}.
\newblock In \emph{Proceedings of the 28th International Conference on Human
  Factors in Computing Systems}, pp.\  1361--1370. ACM, 2010a.

\bibitem{LHK10b}
Leskovec, J., Huttenlocher, D., and Kleinberg, J.
\newblock {Predicting positive and negative links in online social networks}.
\newblock In \emph{Proceedings of the 19th International Conference on World
  Wide Web}, pp.\  641--650. ACM, 2010b.

\bibitem {von07}
Von~Luxburg, U.
\newblock {A tutorial on spectral clustering}.
\newblock \emph{Statistics and Computing}, 17\penalty0 (4):\penalty0 395--416,
  2007.

\end{thebibliography}

\section{Appendix with missing proofs}

\begin{proof}[Proof of Theorem~2]
By Fact~2, it suffices to show that the length of each path used 
for predicting the test edges is bounded by $4k+1$.
For each $T' \in \scT$, we have $D_{T'} \le 2k$, since the height of each subree
is not bigger than $k$. Hence, any test edge incident to vertices
of the same subtree $T' \in \scT$ is predicted (Line $5$ in Figure~1)
using a path whose length is bounded by $2k < 4k+1$. 
Any test edge $(u,v)$ incident to vertices
belonging to two different subtrees $T', T'' \in \scT$ is predicted 
(Line $11$ in Figure~1) using a path whose length is bounded by 
$D_{T'}+D_{T''}+1 \le 2k+2k+1 = 4k+1$,
where the extra $+1$ is due to the query edge $(i',i'')$ connecting $T'$ to $T''$
(Line $9$ in Figure~1).

In order to prove that $|V|-1 + \frac{|V|^2}{2k^2}+\frac{|V|}{2k}$ is an upper bound on
the query set size, observe that each query edge either belongs to $T$ or connects 
a pair of distinct subtrees contained in $\scT$. The number of edges in $T$ is
$|V|-1$, and the number of the remaining query edges is bounded by
the number of distinct pairs of subtrees contained in $|\scT|$, which can be
calculated as follows.
First of all, note that only the last subtree returned by $\tp$ may have 
a height smaller than $k$, all the others must have height $k$.
Note also that each subtree of height $k$ must contain at least
$k+1$ vertices of $V_T$, while the subtree of $\scT$ having height 
smaller than $k$ (if present) must contain at least one vertex. 
Hence, the number of distinct pairs of subtrees contained in $\scT$ can be upper 
bounded by 
\[
\frac{|\scT|(|\scT|-1)}{2} 
\le
\frac{1}{2} \Bigl(\frac{|V|-1}{k+1}+1\Bigr) \Bigl(\frac{|V|-1}{k+1}\Bigr) 
\le 
\frac{|V|^2}{k^2} + \frac{|V|}{k}~.
\] 
This shows that the query set size cannot be larger than $|V|-1 + \frac{|V|^2}{2k^2} + \frac{|V|}{2k}$.

Finally, observe that $D_T \leq 2D_G$ because of the breadth-first visit generating $T$.
If $D_T \leq k$, the subroutine $\tp$ is invoked only once, and the algorithm
does not ask for any additional label of $E_G \setminus E_T$
(the query set size equals $|V|-1$). In this case $\E M$ is clearly upper bounded
by $2D_G\,p|E|$.
\end{proof}

\begin{proof}[Proof of Theorem~3]
In order to prove the claimed mistake bound, it suffices to show that each test edge is predicted
with a path whose length is at most $5$. This is easily seen by the fact that
summing the diameter of two stars plus the query edge $(i',i'')$ that connects them
is equal to $2+2+1=5$, which is therefore the diameter of the tree made up by two stars 
connected by the additional query edge.

We continue by bounding from the above the query set size.
Let $S_j$ be the $j$-th star returned by the $j$-th call to \es.  
The overall number of query edges can be bounded by
$|V|-1 + z$, where $|V|-1$ serves as an upper bound on the number of 
edges forming all the stars output by \es, and $z$ is the sum over $j=1, 2, \ldots$
of the number of stars $S_{j'}$ with $j' > j$ (i.e., $j'$ is created later than $j$)
connected to $S_j$ by at least one edge. 

Now, for any given $j$, the number of stars $S_{j'}$ with $j' > j$ connected to $S_j$ by at least 
one edge cannot be larger that $\min\{|V|, |V_{S_j}|^2\}$.
To see this, note that if there were a leaf $q$ of $S_j$ connected to more than $|V_{S_j}|-1$
vertices not previously included in any star, then \es\ would have returned a star centered in $q$ instead.
The repeated execution of \es\ can indeed be seen as partitioning $V$. 
Let $\scP$ be the set of all partitions of $V$. With this notation in hand, 
we can bound $z$ as follows:
%
%
\begin{equation}\label{e:z}
z \le \max_{P \in \scP} \sum_{j=1}^{|P|} \min\bigl\{z^2_j(P),|V|\bigr\}
\end{equation}
where $z_j(P)$ is the number of nodes contained in the the $j$-th element of 
the partition $P$, corresponding to the number of nodes in $S_j$.
Since $\sum_{j = 1}^{|P|} z_j(P) = |V|$ for any $P \in \scP$, it is easy
to see that the partition $P^*$ maximizing the above expression is such that
$z_j(P^*) = \sqrt{|V|}$ for all $j$, implying $|P^*| = \sqrt{|V|}$.
%
%
%
%
%
%
We conclude that the query set size is bounded by $|V|-1 + |V|^{\frac{3}{2}}$, as claimed.
\end{proof}

\begin{proof}[Proof of Theorem~4]
If the height of $T$ is not larger than $k$, then \tp\ is invoked only once and
$\scT$ contains the single tree $T$.
The statement then trivially follows from the fact that
the length of the longest path in $T$ cannot be larger than
twice the diameter of $G$. Observe that in this case $|V_{G'}|=1$.

We continue with the case when the height of $T$ is larger than $k$.
We have that the length of each path used in the prediction phase
is bounded by $1$ plus the sum of the diameters of two trees of $\scT$. Since these two trees are not higher than $k$, the mistake bound follows from Fact~2.

Finally, 
we combine the upper bound on the query set size in the statement of 
Theorem~3 with the fact that each vertex of $V_{G'}$ 
corresponds to a tree of $\scT$ containing at least $k+1$ vertices of $G$. This 
implies $|V_{G'}|\le \frac{|V|}{k+1}$, and the claim on the query set size of \algst\ follows.
\end{proof}

\begin{proof}[Proof of Theorem~5]
A common tool shared by all three implementations is a preprocessing step.

Given a subtree $T'$ of the input graph $G$
we preliminarily perform a visit of all its vertices (e.g., a depth-first visit) 
tagging each {\em node} by a binary label $y_i$ as follows. 
We start off from an arbitrary node $i \in V_{T'}$, and tag it $y_i=+1$. 
Then, each adjacent vertex $j$ in $T'$ is tagged by $y_j = y_i \cdot Y_{i,j}$. 
The key observation is that, after all nodes in $T'$ have been labeled this way, 
for any pair of vertices $u,v \in V_{T'}$
we have $\pi_{T'}(i,j) = y_i \cdot y_j$, i.e., we can easily compute the parity
of $\path_{T'}(u,v)$ in {\em constant} time. The total time taken for labeling
all vertices in $V_{T'}$ is therefore $\scO(|V_{T'}|)$.

With the above fast tagging tool in hand, we are ready to sketch the implementation details 
of the three algorithms.

\textbf{Part 1.} We draw the spanning tree $T$ of $G$ and tag as described above all its 
vertices in time $\scO(|V|)$. 
We can execute the first $6$ lines of the pseudocode in Figure~5 in time $\scO(|E|)$ 
as follows.
For each subtree $T_i \subset T$ rooted at $i$ returned by \tp, 
we assign to each of its nodes a pointer to its root $i$. 
This way, given any pair of vertices, we can now determine whether 
they belong to same subtree in constant time. 
We also mark node $i$ and all the leaves of each subtree. This operation is useful
when visiting each subtree starting from its root. 
Then the set $\scT$ contains just the roots of all the subtree returned by \tp. 
This takes $\scO(|V_T|)$ time. 
For each $T' \in \scT$ we also mark each edge in $E_{T'}$ so as to determine 
in constant time whether or not
it is part of $T'$. We visit the nodes of each subtree $T'$ whose root is in $\scT$, 
and for any edge $(i,j)$ connecting two vertices of $T'$, we predict in constant time 
$Y_{i,j}$ by $y_i \cdot y_j$.
It is then easy to see that the total time it takes to compute these predictions 
on all subtrees returned by \tp\ is $\scO(|E|)$.

To finish up the rest, we allocate a vector $\bv$ of $|V|$ records, each record $v_i$ storing only
one edge in $E_G$ and its label.
For each vertex  $r \in \scT$ we repeat the following steps.
We visit the subtree $T'$ rooted at $r$.
For brevity, denote by $\rt(i)$ the root of the subtree which $i$ belongs to.
%
For any edge connecting the currently visited node 
$i$ to a node $j \not\in V_{T'}$, we perform the following operations: 
if $v_{\rt(j)}$ is empty,
we query the label $Y_{i,j}$ and insert edge $(i,j)$ 
together with $Y_{i,j}$ in $v_{\rt(j)}$. 
If instead $v_{\rt(j)}$ is not empty, we set $(i,j)$ to be part
of the test set and predict its label as 
\[
{\hat Y}_{i,j} \leftarrow \pi_{T}(i,z') \cdot Y_{z',z''} \cdot \pi_{T}(z'',j) 
= y_i \cdot y_{z'} \cdot Y_{z',z''} \cdot y_{z''} \cdot y_{j},
\]
where $(z',z'')$ is the edge contained in $v_{\rt(j)}$. We mark each predicted edge so as to avoid
to predict its label twice.
We finally dispose the content of vector $\bv$.

The execution of all these operations takes time overall linear in $|E|$, 
thereby concluding the proof of Part 1.

\textbf{Part 2.} We rely on the notation just introduced. 
We exploit an additional data structure, which takes extra $\scO(|V|)$ space. 
This is a heap $H$ whose records $h_i$ contain references to vertices $i \in V$. Furthermore,
we also create a link connecting $i$ to record $h_i$. 
The priority key ruling heap $H$ is the degree of each vertex referred to by its records. 
With this data structure in hand, we are able to find the vertex
having the highest degree (i.e., the {\em top} element of the heap) in constant time. 
The heap also allows us to execute in logarithmic time a {\em pop} operation, which eliminates 
the {\em top} element from the heap.

In order to mimic the execution of the algorithm, 
we perform the following operations. We create a star $S$ centered
at the vertex referred to by the top element of $H$ connecting it with all the adjacent vertices
in $G$. We mark as ``not-in-use'' each leaf of $S$. Finally, we eliminate the element pointing to the 
center of $S$ from $H$ (via a pop operation) and create a pointer from each leaf of $S$ to its central 
vertex.
We keep creating such star graphs until $H$ becomes empty. Compared to the creation of the first star, 
all subsequent stars essentially require the same sequence of operations.
The only difference with the former is that when the top element of
$H$ is marked as not-in-use, we simply pop it away. This is because 
%
%
any new star that we create is centered at a node that is not part of any previously generated star.
The time it takes to perform the above operations is $\scO(|V| \log |V|)$.

Once we have created all the stars, we predict all the test edges the very same
way as we described for \alg\ (labeling the vertices of each star, using a set $\scT$ 
containing all the star centers and the vector $\bv$ for computing the predictions). 
Since for each edge we perform only a constant number of operations, 
the proof of Part 2 is concluded.

\textbf{Part~3.} \algst(k)\ can be implemented by combining
the implementation of \alg\ with the implementation of \algs.
In a first phase, the algorithm works as \alg, creating a set $\scT$ containing
the roots of all the subtrees with diameter bounded by $k$.
We label all the vertices of each subtree and create a pointer from each node
$i$ to $\rt(i)$. Then, we visit
all these subtrees and create a graph $G' = (V',E')$ having the following properties:
$V'$ coincides with $\scT$, and there exists an edge $(i,j) \in E'$ 
if and only if there exists at least one edge connecting the subtree rooted
at $i$ to the subtree rooted at $j$. 
We also use two vectors $\bu$ and $\bu'$, both having $|V|$ components, 
mapping each vertex in $V$ to a vertex in $V'$, and viceversa. 
Using $H$ on $G'$, the algorithm splits the whole set of subtrees into stars of 
subtrees. The root of the subtree which is the center of each star
is stored in a set $\scS \subseteq \scT$. In addition to these operations,
we create a pointer from each vertex of $S$ to $r$. For each $r \in \scS$, 
the algorithm predicts the labels of all edges connecting pairs of vertices 
belonging to $S$ using a vector $\bv$ as for \alg.  
Then, it performs a visit of $S$ for the purpose of relabeling all its vertices 
according to the query set edges that connect the subtree in the center of $S$ with 
all its other subtrees.
Finally, for each vertex of $\scS$, we use vector $\bv$ as in \alg\ and \algs\ for
selecting the query set edges connecting the stars of subtrees so created and
for predicting all the remaining test edges.

Now, $G'$ is a graph that can be created in $\scO(|E|)$ time.
The time it takes for operating with $H$ on $G'$ is 
$\scO(|V'| \log |V'|) = \scO\Bigl(\frac{|V|}{k} \log \frac{|V|}{k}\Bigr)$, 
the equality deriving from the fact that each subtree
with diameter equal to $k$ contains at least $k+1$ vertices, thereby making 
$|V'| \leq \frac{|V|}{k}$. Since the remaining operations need constant time per 
edge in $E$, this concludes the proof.
\end{proof}

\end{document}